\documentclass{tlp}

\usepackage{amsmath}
\usepackage{graphicx}
\usepackage{multirow}

\usepackage{amssymb}
\usepackage{hyperref}
\usepackage{booktabs}
\usepackage[english]{babel}

\newtheorem{proposition}{Proposition}

\usepackage{graphicx}
\usepackage{mathptmx}
\usepackage{tikz-qtree}
\usepackage{filecontents}
\usepackage{csvsimple}
\usepackage{amsfonts,amsmath}
\usepackage{subfig} 
\usepackage{url}
\usepackage{verbatim}
\usepackage{color}
\usepackage{amsmath}
\usepackage{proof}
\usepackage{listings}
\usepackage{xcolor}
\usepackage{fancyvrb}
\renewcommand{\phi}{\varphi}

\input prolog.tex

\definecolor{lgray}{gray}{0.95}
\definecolor{lblue}{rgb}{0.90,0.90,1.00}
\definecolor{lyellow}{rgb}{1.00,1.00,0.70}

\newtheorem{prop}{Proposition}

\newtheorem{ex}{Example}

\DefineVerbatimEnvironment{codex}
{Verbatim}
{fontsize=\small, frame=single, formatcom=\color{blue}}

\newcommand{\BI}[0]{\begin{itemize}}
\newcommand{\EI}[0]{\end{itemize}}
\newcommand{\I}[0]{\item}
\newcommand{\BE}[0]{\begin{enumerate}}
\newcommand{\EE}[0]{\end{enumerate}}

\newcommand{\BX}[0]{\begin{ex}}
\newcommand{\EX}[0]{\end{ex}}
\newcommand{\BV}[0]{\begin{verbatim}}
\newcommand{\EV}[0]{\end{verbatim}}

\newcommand{\BP}[0]{\begin{prop}}
\newcommand{\EP}[0]{\end{prop}}

\newcommand{\BEQ}{\begin{equation}}
\newcommand{\EEQ}{\end{equation}}

\newcommand{\BC}[0]{\begin{center}}
\newcommand{\EC}[0]{\end{center}}

\newcommand{\BF}[0]{\begin{filecontents*}{data.csv}}

\newcommand{\BQ}[0]{\color{blue}\begin{quote}}
\newcommand{\EQ}[0]{\end{quote}\color{black}}

\def \bscale1 {0.25}
\def \bscale {0.25}

\usepackage{lineno} 

\begin{document}


\title[Modeling Next-Token Prediction as Left-Nested Intuitionistic Implication]
{Modeling Next-Token Prediction as Left-Nested Intuitionistic Implication}

\begin{authgrp}
\author{\sn{Paul} \gn{Tarau}}
\affiliation{University of North Texas}
\end{authgrp}


\lefttitle{Paul Tarau}

\jnlPage{x}{xx}
\jnlDoiYr{xxxx}
\doival{10.1017/xxxxx}

\maketitle

\begin{abstract}
We introduce the \emph{Arrow Language Model}, a neural architecture derived from an intuitionistic-logic interpretation of next-token prediction. Instead of representing tokens as additive embeddings mixed by attention, we encode a prefix as a \emph{left-nested implication chain} whose structure preserves order through non-commutative composition. Next-token prediction corresponds to \emph{modus ponens}, and sequence processing becomes constructive proof extension under the Curry--Howard correspondence. Our Prolog-based specialized theorem provers validate  fundamental properties of the neural models, among which relations between commutative vs. non-commutative sequencing  and single-token vs. multi-token prediction choices.
We show that a neural architecture equivalent to multiplicative RNNs
arises naturally from a proof-theoretic interpretation of next-token
prediction as nested intuitionistic implication,
we present a practical low-rank neural realization
and position the model relative to Transformers and state-space models.

{\bf Keywords:}
logic-based derivation of neural architectures, intuitionistic implicational logic,
token-as-operator neural models, state-space models,
alternatives to transformer-based foundational models.
\end{abstract}

\section{Introduction}

Foundational Generative AI models involve  mechanisms like next token(s) prediction
that naturally benefit
from unsupervised training on very large data sets
at a scale closely matching the total multi-modal content of the internet.

The resulting distribution, encapsulated in the parametric memory of
the model, {\em constructively} completes the initial sequence
derived from the prompt and subsequently generated tokens with
a set of probability-weighted next token candidates.

The underlying logic of the derivation of the next token
hints to an intuitionistic implicational logic  proof of  $q$ seen as an implication from assumptions in context $E$,
$E \rightarrow q$.
This is can be stated as
{\em finding a way to transform a proof of an assumption context $E$ into a proof of
    a consequence $q$}.
In the neural world this materializes as
a transformation of the overall statistical evidence for
$E$ into a predicted probability for $q$.

In the case of the transformer models \cite{transfo},
order information is learned via dedicated positional encodings
that complement the parallel attention heads that learn
deep dependencies relating to the ``past'' token sets
while having the set of ``future tokens'' masked.
 This is a highly effective approach, but it obscures the logical structure of sequence prediction and provides little explanatory connection to symbolic reasoning.

We will explore here an alternative view: next-token prediction as implication completion. We ask whether the task of predicting the next word in a sequence can be understood as extending a proof in intuitionistic propositional logic—and whether this interpretation leads naturally to a viable neural architecture.

Our answer is affirmative. Starting from a symbolic encoding of sequences as left-nested intuitionistic implications, we derive a recurrent neural model in which each token acts as an operator transforming a proof state. The resulting architecture, which we call the {\em Arrow Language Model}, replaces similarity-based attention with non-commutative operator composition as the carrier of both semantic content and order.

Our logic-based insights will also be extended to cover multi-token prediction
and it naturally relates to type inference and type inhabitation via the Curry-Howard isomorphism.

The rest of the paper is organized as follows.
Section \ref{tools} introduces our theorem provers for propositional implicational intuitionistic logic.
Section \ref{intu} justifies our choice of the logic formalism modeling content and order
in the next token prediction.
Section \ref{theo} illustrates instances of valid formulas relating left-nested implicational sequences,
relevant  for the next word(s) prediction task.
Section \ref{info} introduces an information retrieval mechanism by using implicational logic subformulas
derived from a corpus split into sentences.
Section \ref{neural} describes the a neural realization of the derived {\em Arrow Architecture}.
Section \ref{exper} describes our experiments involving validation of our logic-based algorithms and their neural equivalents.
Section \ref{disc} discusses limitations and future work directions.
Section \ref{rel} overviews related work.
Section \ref{conc} concludes the paper.

Our full open-source code is available at \url{https://github.com/ptarau/nextword} .

The literate Prolog code extracted from this paper is at \url{https://github.com/ptarau/nextword/blob/main/nextword.pro} .

\section{Our tools :  implicational intuitionistic logic theorem provers}\label{tools}

We will start by motivating  our choice of a purely symbolic logic representation that does  not involve reliance on numerical implementations of ordering.

Implicational propositional logic allows reasoning, when seen as
a Hilbert system with modus-ponens as its only inference rule:
\[
\begin{gathered}
    p \to q \\
    p \\
    \hline
    \therefore q
\end{gathered}
\]
and with as little as  two axioms corresponding via the Curry-Howard isomorphism
to combinators K and S:
\begin{codex}
    K combinator:
    p->q->p

    S combinator:
    (p->q->r)->(p->q)->p->r
\end{codex}

For the same formalism, when seen via the sequent calculus, simple and efficient theorem provers
become available.

We will introduce them next as they will be used to support our logical inference mechanisms.

\noindent
The predicate {\tt iprove} implements Roy Dyckhoff's {\bf LJT} calculus, restricted to the implicational fragment of propositional intuitionistic logic, proven as sound and complete in \cite{dy1}.
\vskip 0.5cm

{\large
    \noindent
    \begin{math}
        LJT_1:~~~~\frac{~}{A,\Gamma ~\vdash~ A}\\\\
        LJT_2:~~~~\frac{A,\Gamma ~\vdash~ B}{\Gamma ~\vdash~ A\rightarrow B}\\\\
        LJT_3:~~~~\frac{B,A,\Gamma ~\vdash~ G}{A \rightarrow B,A,\Gamma ~\vdash~ G}\\\\ 
        LJT_4:~~~~\frac{D \rightarrow B,\Gamma ~\vdash~ C \rightarrow D ~~~~ B,\Gamma ~\vdash~ G}
        { \left( C \rightarrow D \right) \rightarrow B,\Gamma ~\vdash~ G }\\
    \end{math}
}
Termination is ensured as one can identify a multiset
ordering-based size definition that decreases after each step \cite{dy1}.
~\\

\begin{code}
iprove(T):-iprove(T,[]).

iprove(A,Vs):-memberchk(A,Vs),!.
iprove((A->B),Vs):-!,iprove(B,[A|Vs]).
iprove(G,Vs1):-
    select((A->B),Vs1,Vs2),
    iprove_imp(A,B,Vs2),
    !,
    iprove(G,[B|Vs2]).

iprove_imp((C->D),B,Vs):-!,iprove((C->D),[(D->B)|Vs]).
iprove_imp(A,_,Vs):-memberchk(A,Vs).
\end{code}

The predicate {\tt lprove} adds generation of a proof term by introducing
lambda terms {\tt l(X,E)} with lambda variable {\tt X} and expression {\tt E},
and {\tt a(A,B)} representing application of a source {\tt A} to  target {\tt B}.

\begin{code}
lprove(T,X):-lprove(X,T,[]),!.

lprove(X,A,Vs):-memberchk(X:A,Vs),!. 
lprove(l(X,E),(A->B),Vs):-!,lprove(E,B,[X:A|Vs]).  
    lprove(E,G,Vs1):-
    select(S:(A->B),Vs1,Vs2),       
    lprove_imp(T,A,B,Vs2),          
    !,
    lprove(E,G,[a(S,T):B|Vs2]).     

lprove_imp(l(X,E),(C->D),B,Vs):-!,lprove(E,(C->D),[X:(D->B)|Vs]).
lprove_imp(E,A,_,Vs):-memberchk(E:A,Vs).
\end{code}

\section{Content and order in propositional intuitionistic implicational logic}\label{intu}

Next we will express the ``next token prediction'' as logical inference in implicational intuitionistic logic.
Thus, we would like, from an implicational representation of ``the cat'' and a representation of ``the cat sits'' , to infer the next token ``sits'', by just using {\em modus ponens}. We would also like
the inference steps to be directly verifiable by our theorem provers.

\subsection{Order Preservation by Left-Nested Implication}

Each implication introduces a higher-order dependency on the entire previous structure. Any permutation alters the functional type of the antecedent, changing the set of inhabitants.
Let $w_1,\dots,w_n$ be atomic propositions and define
\[
L_n = ((((w_1 \to w_2) \to w_3) \to \dots) \to w_n).
\]
Then $L_n$ is not invariant under permutation of the $w_i$ in intuitionistic logic.

\BX Counterexample for order invariance in left-nested implication chain.
\begin{codex}
?- iprove(((p->q)->r) -> ((q->p)->r)).
false.
\end{codex}
\EX

By contrast, right-nested implication $R_n$ collapses to a conjunction-to-result form $C_n$
and is $w_1 \ldots w_{n-1}$ permutation-invariant given that conjunction is commutative.

\[
R_n = w_1 \to w_2 \to w_3 \to \dots \to w_{n-1} \to w_n.
\]

\[
C_n = w_1 \land w_2 \land w_3 \land \dots \land w_{n-1} \to w_n.
\]

\BX Order invariance for right-nested implication:
\begin{codex}
  ?- iprove((p->q->r) -> (q->p->r)).
  true.
\end{codex}
\EX

\BX
Given the left-nested implicational representation of the tokens in a sentence
and its infixes ensure inference of the next word by applying just modus ponens:

\begin{codex}
((the->cat)->sits)
(the->cat)
------------------ modus ponens
sits
\end{codex}
\noindent
or equivalently, by reflecting {\em modus ponens} as an implication

\begin{codex}
?- iprove(((the->cat)->sits) -> (the->cat) -> sits).
true.
\end{codex}
\EX

Note that expressing {\em modus ponens}
in terms of implications
 makes it testable with a purely
implicational prover, given the equivalence between
$(p \land (p \to q)) \to q$ and $p \to (p \to q) \to r$.

\subsection{From sequences to implication chains}

As our tokenized input is expressed as a Prolog list of words
we will need to convert them as left-nested implications.

\begin{code}
list2impl([],[]).
list2impl([X],X).
list2impl([X,Y|Xs],R):-seq2left_heavy([Y|Xs],X,R).

seq2left_heavy([],End,End).
seq2left_heavy([X|Xs],Chain,End):-seq2left_heavy(Xs,(Chain->X),End).
\end{code}

\BX
\begin{codex}
?- list2impl([the,cat,sits,on,the,map],R).
R = (((((the->cat)->sits)->on)->the)->map).
\end{codex}
\EX

\subsection{Implication Depth and Proof Steps}\label{depth}

We now make explicit the correspondence between symbolic implication depth and  depth of recurrence.

\begin{proposition}[Implication Depth Corresponds to Recurrence Steps]
    Let $L_n = ((((w_1 \to w_2) \to \dots) \to w_n)$ be a left-nested implication chain.
    Then $L_n$ requires exactly $n-1$ successive implication introductions and eliminations.
\end{proposition}

\begin{proof}
    Each nesting level introduces one implication whose antecedent is the entire previous chain.
    Constructively, forming or eliminating this implication requires one application of modus ponens.
    Thus, evaluating or extending $L_n$ proceeds stepwise, with no parallel elimination of implications.
    This enforces a strict sequential dependency structure of depth $n-1$.
\end{proof}

\subsection{Next-Token Prediction as Modus Ponens}

Given prefix implication $I_p$ and full implication $I_f$,
\[
I_f = (I_p \to w)
\]
producing $w$ is an instance of modus ponens:
\[
I_f \to (I_p \to I_f) \to w.
\]
Thus, next-token prediction  corresponds to constructive proof completion in intuitionistic implicational logic.

\subsection{Curry--Howard Interpretation}

Under the Curry--Howard isomorphism, propositions correspond to types,
and proofs correspond to $\lambda$-terms inhabiting those types.
From this perspective, next-token prediction can be viewed as the
construction of a term inhabiting an implicational type determined by
the prefix.

\begin{center}
    \begin{tabular}{lll}
        \toprule
        Intuitionistic Implicational Logic
        & $\lambda$-Calculus (Curry--Howard)
        & Language Modeling \\
        \midrule
        Proposition
        & Type of a $\lambda$-term
        & Token as operator \\

        Implication $A \to B$
        & Function type $A \to B$
        & State transition operator \\

        Proof of $A \to B$
        & $\lambda$-abstraction
        & Learned token operator \\

        Proof term
        & $\lambda$-term
        & Prefix representation \\ 

        Modus ponens
        & Function application
        & Next-token inference \\

        Normalization of proofs
        & $\beta$-reduction
        & Recurrent state update \\
        \bottomrule
    \end{tabular}
\end{center}

In this interpretation, a sentence prefix corresponds to a partially
applied $\lambda$-term whose type is a left-nested implication.
Predicting the next token amounts to supplying an argument that enables
a further function application, i.e., an instance of modus ponens.
The corresponding neural architecture can thus be seen as learning a numeric
realization of proof construction, where recurrent state updates
correspond to successive $\beta$-reductions in the associated
$\lambda$-term.


\section{Theorems about single next word and multi-word predictions}\label{theo}

We will illustrate our intuitionistic axioms and theorems as an implicational formula, followed by its tree representation and its proof term as an expression
in $\lambda$-calculus .
\subsection{Axioms}

\BX K combinator:

\noindent \verb|p->q->p|\Tree [.$\rightarrow$ [.p ] [.$\rightarrow$ [.q ] [.p ]  ]  ]
$ \lambda X. \lambda Y.X $
\EX

\BX S combinator:

\noindent  \verb|(p->q->r)->(p->q)->p->r|\Tree [.$\rightarrow$ [.$\rightarrow$ [.p ] [.$\rightarrow$ [.q ] [.r ]  ]  ] [.$\rightarrow$ [.$\rightarrow$ [.p ] [.q ]  ] [.$\rightarrow$ [.p ] [.r ]  ]  ]  ]
$ \lambda X. \lambda Y. \lambda Z.((X ~Z )~(Y ~Z ))$
\EX

\BX Modus Ponens as implicational formula:

\noindent  \verb|p -> (p->q) -> q|\Tree [.$\rightarrow$ [.p ] [.$\rightarrow$ [.$\rightarrow$ [.p ] [.q ]  ] [.q ]  ]  ]
$ \lambda X. \lambda Y.(Y ~X )$
\EX

\subsection{Theorems and their relevance for single- and multi-word prediction}
\label{sec:theorems-relevance}

We present a small collection of valid implicational instances (verified with \texttt{iprove})
that clarify how a prefix can support both \emph{single next-word inference} and \emph{multi-word
    completion}. The unifying idea is that left-nested implication chains naturally induce
\emph{continuations}: formulas of the form $A \to B$ that, under Curry--Howard, correspond to
functions mapping a proof of $A$ into a proof of $B$. In the language-modeling reading, such
continuations suggest how one may construct a mechanism that extends a prefix one token at a time.

\BX Left nested implies right nested implication chain:

\noindent \verb|((p->q)->r) -> (p->q->r)|
\Tree [.$\rightarrow$ [.$\rightarrow$ [.$\rightarrow$ [.p ] [.q ]  ] [.r ]  ] [.$\rightarrow$ [.p ] [.$\rightarrow$ [.q ] [.r ]  ]  ]  ]
$ \lambda X.\lambda Y.\lambda Z.\big(X~(\lambda U.\lambda V.Z)\big)$
\EX

\paragraph{Relevance.}
This theorem expresses a \emph{currying} principle: a context that consumes a composite implication
$(p\!\to\!q)$ to produce $r$ can be transformed into a predictor that consumes $p$ and then $q$ to
produce $r$.
For prediction, this captures a basic bridge between two views:
\begin{itemize}
    \item \emph{single-step view:} the prefix determines a rule taking a composite premise to the next
    token,
    \item \emph{incremental view:} the same information can be reorganized into a step-by-step process
    consuming premises one at a time.
\end{itemize}
Thus, even though the left-nested encoding is order-sensitive, it still supports incremental
generation by exposing a chain of smaller continuations.

\BX Sequence implies Block to Block:

\noindent  \verb|(((p->q)->a)->b) -> (p->q)->(a->b)|
\Tree [.$\rightarrow$ [.$\rightarrow$ [.$\rightarrow$ [.$\rightarrow$ [.p ] [.q ]  ] [.a ]  ] [.b ]  ] [.$\rightarrow$ [.$\rightarrow$ [.p ] [.q ]  ] [.$\rightarrow$ [.a ] [.b ]  ]  ]  ]
$ \lambda X. \lambda Y. \lambda Z.(X ~ \lambda U. \lambda V.Z )$
\EX

\paragraph{Relevance.}
This theorem shows how a longer dependency can be refactored into a \emph{block-to-block}
continuation.
If, given $(p\!\to\!q)$, the presence of $a$ suffices to obtain $b$, then one can extract a direct
implication $a\!\to\!b$ that is valid under the same contextual constraint $(p\!\to\!q)$.
In prediction terms, this suggests the following decomposition:
the prefix may establish a contextual capability (here represented by $(p\!\to\!q)$), after which
generation can proceed via simpler local rules linking intermediate tokens or phrases.
This provides a logical template for multi-word prediction: the model need not ``re-derive'' the whole
prefix at every step; instead, it may compile the prefix into reusable continuations.

\BX Nested implies simple:

\noindent  \verb|((E->p)->q)->(p->q)|
\Tree [.$\rightarrow$ [.$\rightarrow$ [.$\rightarrow$ [.E ] [.p ]  ] [.q ]  ] [.$\rightarrow$ [.p ] [.q ]  ]  ]
$ \lambda X.\lambda Y.\big(X~(\lambda Z.\lambda U.Y)\big)$
\EX

\paragraph{Relevance.}
Here $E$ stands for an abstract \emph{environment} (or background context) that is held fixed.
The theorem states that if $q$ follows from the assumption $(E\!\to\!p)$, then $q$ also follows from
$p$ alone (with the same ambient $E$).
As a prediction principle, this isolates a common phenomenon:
once a prefix has established a rich background context, deriving the next token may require only a
small local cue $p$ to trigger the appropriate continuation.
In other words, $E$ can be treated as an implicit parameter to the inference, while the explicit
premise $p$ selects the next step.
This theorem suggests how a predictive mechanism could separate \emph{stored context} from
\emph{local triggers}.

\BX Multiple next words :

\noindent  \verb|E -> (((E->a)->b)->c) -> ((a->b)->c)|

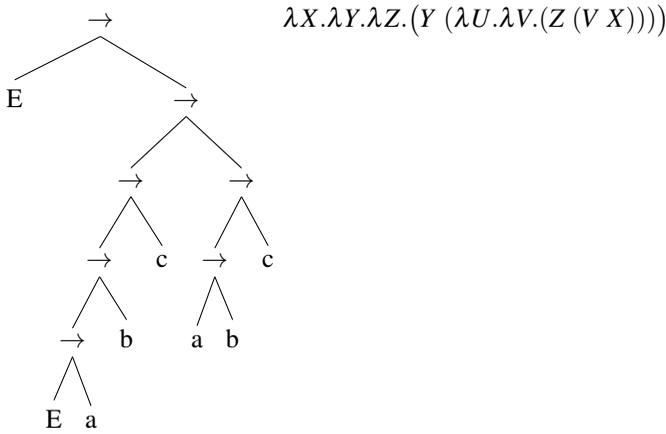
\Tree [.$\rightarrow$ [.E ] [.$\rightarrow$ [.$\rightarrow$ [.$\rightarrow$ [.$\rightarrow$ [.E ] [.a ]  ] [.b ]  ] [.c ]  ] [.$\rightarrow$ [.$\rightarrow$ [.a ] [.b ]  ] [.c ]  ]  ]  ]
$ \lambda X.\lambda Y.\lambda Z.\big(Y~(\lambda U.\lambda V.(Z~(V~X)))\big)$
\EX

\paragraph{Relevance.}
This formula provides a schematic link between single-step and multi-step completion.
Given a fixed context $E$ and a ``two-step'' implicational chain from $(E\!\to\!a)$ to $b$ and then to
$c$, we can derive a continuation $(a\!\to\!b)\!\to\!c$ that no longer mentions $E$ explicitly.
Intuitively, $E$ can be ``compiled into'' the continuation rule, leaving a purely sequential
operator chain over the intermediate tokens $a,b,c$.

For multi-word prediction, this suggests an important pattern:
the prefix context can be used to construct a continuation that generates several future tokens by
iterated application.
Such a continuation is precisely the kind of object that, under Curry--Howard, corresponds to a
composable function, indicating how multi-token generation can be framed as repeated implication
elimination.

\paragraph{Summary.}
These theorems collectively motivate the view that next-token prediction can be expressed in the
implicational calculus not only as a single modus-ponens step, but also as the construction of
\emph{continuation operators} that support multi-step completion. This sets the stage for a later
section in which we show how such continuation structure can be realized computationally.

\section{Information retrieval with logic formulas}\label{info}

When our provers work on the sequent calculus representation of the assumptions
needed to prove a theorem, they are represented as a list.
This is needed as the proof procedure progressively reduces
this list of assumptions to a simpler form.
As we can reason exclusively with modus ponens on our
left-nested implication chains, we can replace with a fact database
representation this list representation of our assumptions
(each derived from a sentence on a document).
We will add each of the assumptions represented
as a left-nested implication into the dynamic predicate \verb~isent/1~, as shown in the following example:

\begin{codex}
isent((((((the->cat)->sits))->on)->the)->mat).
...
\end{codex}

To implement an information retrieval mechanism we would want a query matching
a subformula (actually seen as a prefix of a suffix in the corresponding sequence)
as a retriever of the sentence(s) it originates from.

\subsection{Assuming the subformulas that imply the full sequence representation}

We start by extracting the set of prefixes of the suffixes
of a formula, each such fragment corresponding to a sequence of tokens,
that, when used as a query, should retrieve the formula.

\begin{code}
ipref(X,X).
ipref(Xs->_,Ys):-
  ipref(Xs,Ys).
\end{code}

\begin{code}
isuff(X,Y):-isuff0(X,Y).
isuff(X,X).

isuff0(_->X,X).
isuff0(Xs->X,(Ys->X)):-
  isuff0(Xs,Ys).
\end{code}

\begin{code}
isufpref-->isuff,ipref.
iprefsuf-->ipref,isuff.
\end{code}

\proposition{
 If the depth of the left-nested formula is $n$, {\tt isufpref} will generate $(n * (n+1) / 2 $)
 terms.
}

\BX Results of {\tt isufpref} on a small left-extend formula:
\begin{codex}
?- isufpref((((the->little)->cat)->sits),R).
    R = sits ;
    R = (cat->sits) ;
    R = cat ;
    R = ((little->cat)->sits) ;
    R = (little->cat) ;
    R = little ;
    R = (((the->little)->cat)->sits) ;
    R = ((the->little)->cat) ;
    R = (the->little) ;
    R = the.
\end{codex}
\EX

Note that this set is a superset of the set of subformulas as defined in sequent calculus
formalisms.

To implement a basic information retrieval mechanism we would
like each of these fragments to fetch the corresponding sentence
from our database of implicational formulas representing them.
We could implement this in Prolog as assumptions added
to our database-represented set as in:
\begin{codex}
isent((the->cat) -> (the->cat)->sits).
...
\end{codex}
but to avoid the quadratic increase in size, we could just
compute the \verb~Fragment->Formula~ assumptions on the fly,
done efficiently by the predicate {\tt isufpref/2}, guided
by unification with the query formula.

We will do the same when training our neural net in {\tt arrow.py},
generate them on the fly as an in-memory augmentation of the training set:
\begin{codex}
def sufpref(xs: List[str]) -> Iterable[List[str]]:
  n = len(xs)
  for i in range(n):
    for j in range(i + 1, n + 1):
      yield xs[i:j]
\end{codex}

\subsection{Information retrieval with subformulas as queries}

\subsubsection{Creating the database of implicational formulas}
\begin{code}
:-dynamic isent/1.

store_impls(Sents):-
  retractall(isent(_)),
  member(Sent,Sents),
  distinct(Impl,sent2impl(Sent,Impl)),
  assertz(isent(Impl)),
  fail;true.
\end{code}

\begin{code}
sent2impl(Sent,R):-
  atomic_list_concat(Words,' ',Sent),
  list2impl(Words,R).
\end{code}

\subsubsection{Querying the database of implicational formulas}
The following query-answering predicates implement our
logic-based information retrieval mechanism. 

\begin{code}
qa(TextQuery,ISent):-
    string_lower(TextQuery,Query1),
    atom_string(QueryAtom,Query1),
    atomic_list_concat(Words,' ',QueryAtom),
    lqa(Words,ISent).
\end{code}

We can also query with a list of words some of which could be logic variables
standing for unknown intermediate next words.
\begin{code}
lqa(Words,ISent):-
    list2impl(Words,LeftImplQuery),
    iqa(LeftImplQuery,ISent).
\end{code}

\begin{code}
iqa(LeftImplQuery,ISent):-
    isent(ISent), 
    distinct(ISent,isufpref(ISent,LeftImplQuery)).
\end{code}

\begin{code}
qa(Query):-
    qa(Query,ISent),
    impl2list(ISent,Words),
    atomic_list_concat(Words,' ',Answer),
    writeq(Answer),nl,nl,
    fail;true.
\end{code}

\begin{code}
impl2list(E,Xs):-impl2list(E,Xs,[]).

impl2list(E->X)-->!,impl2list(E),[X].
impl2list(X)-->[X].
\end{code}

\BX
Querying on a toy database.
We create the database with the predicate {\tt store\_impls}.

\begin{code}
store_impls :-
  Sents=['the cat sits on the mat',
    'the dog sits on the log', 'the cat chases the mouse', 'the dog chases the cat'],
  store_impls(Sents),listing(isent/1).
\end{code}

Here is the result of querying it:
\begin{codex}
?- qa('the cat').
  'the cat sits on the mat'
  'the cat chases the mouse'
  'the dog chases the cat'

?- lqa([the,X,chases],R).
  X = cat, R = ((((the->cat)->chases)->the)->mouse) ;
  X = dog, R = ((((the->dog)->chases)->the)->cat) .
\end{codex}

Note also that in the last example we use logic variables for unknown query components.
\EX

\begin{codeh}
/*

?- iprove(((p->q)->r) -> (p->q->r)).
true.

?- lprove(((p->q->r)->(p->q)->p->r),SCombinator).
SCombinator = l(_A, l(_B, l(_C, a(a(_A, _C), a(_B, _C)))))

?- lprove(((p->q)->r) -> (p->q->r), LambdaExpr).
LambdaExpr = l(_A, l(_, l(_B, a(_A, l(_, l(_, _B)))))).

?- sent2impl('the cat sits on the mat',Impl).
Impl = (((((the->cat)->sits)->on)->the)->mat).

?- store_impls.
:- dynamic isent/1.

isent((((((the->cat)->sits)->on)->the)->mat)).
isent((((((the->dog)->sits)->on)->the)->log)).
isent(((((the->cat)->chases)->the)->mouse)).
isent(((((the->dog)->chases)->the)->cat)).

?- qa('the cat').
'the cat sits on the mat'

'the cat chases the mouse'

'the dog chases the cat'

?- lqa([the,X,chases],R).
X = cat,
R = ((((the->cat)->chases)->the)->mouse) ;
X = dog,
R = ((((the->dog)->chases)->the)->cat) .

*/
\end{codeh}


\section{Neural Realization of Implicational Composition}\label{neural}
\label{sec:neural-realisation}

This section describes a concrete neural realization of the left-nested
implicational semantics introduced earlier.
The goal is not to propose a novel neural architecture per se, but to
demonstrate that a well-defined proof-theoretic interpretation leads
naturally to a recurrent, operator-based model closely related to known
multiplicative RNN architectures.

\subsection{Hidden State as Proof Context}

The model maintains a hidden state
\[
h_t \in \mathbb{R}^d
\]
which represents the accumulated proof context after consuming a prefix
of tokens $x_1,\dots,x_t$.
From a Curry--Howard perspective, $h_t$ is a compact numeric encoding of
the assumptions and justifications constructed so far.

The initial state $h_0$ is a learned vector that seeds the computation at
the beginning of each sequence.

\subsection{Tokens as Implication Operators}

Each token $t$ is treated not as a feature vector to be added, but as an
\emph{operator} that transforms the current proof context.
Formally, each token induces a linear operator
\[
M_t : \mathbb{R}^d \rightarrow \mathbb{R}^d
\]
which is applied to the current state.

To make this feasible for large vocabularies, the operator is
parameterized in low-rank form:
\[
M_t = I + U \, \mathrm{diag}(s_t) \, V^\top,
\]
where:
\begin{itemize}
    \item $U, V \in \mathbb{R}^{d \times r}$ are shared trainable matrices,
    \item $s_t \in \mathbb{R}^r$ is a token-specific gate vector,
    \item $r \ll d$ is a small rank hyperparameter.
\end{itemize}

The gate vector $s_t$ is obtained by embedding lookup and passed through
a $\tanh$ nonlinearity to keep operator perturbations bounded.

This construction yields an operator family that is expressive,
parameter-efficient, and explicitly non-commutative.

\subsection{State Update as Modus Ponens}

The recurrent update implements implication elimination (modus ponens)
as operator application:
\[
h_{t+1}
=
\mathrm{LayerNorm}\!\left(
h_t
+
U\big((V^\top h_t) \odot s_t\big)
\right).
\]

This update can be read directly as:
\begin{quote}
    apply the implication introduced by token $t$ to the current proof
    context.
\end{quote}

Because matrix composition is non-commutative in general,
\[
M_a M_b \neq M_b M_a,
\]
the order of tokens is encoded intrinsically by the algebra of operator
composition.
No explicit positional encoding is required in the baseline model.

\subsection{Relation to Left-Nested Implication}

After consuming tokens $x_1,\dots,x_t$, the hidden state corresponds to
the left-nested implication
\[
((((x_1 \rightarrow x_2) \rightarrow x_3) \rightarrow \dots) \rightarrow x_t)
\]
applied to the initial seed $h_0$.

Each recurrence step corresponds to one implication elimination, in
direct correspondence with the proof-theoretic structure described in
subsection~\ref{depth}.

\subsection{Output Projection and Training Objective}

After each update, the model predicts the next token using a standard
linear readout:
\[
\mathrm{logits}_t = W_{\text{out}} h_t,
\]
where $W_{\text{out}} \in \mathbb{R}^{|\mathcal{V}| \times d}$ projects the
current proof context into vocabulary logits.

Training uses standard next-token cross-entropy with teacher forcing.
Importantly, while the \emph{internal representation of composition}
differs from Transformer models, the \emph{external training objective}
remains unchanged.
This allows direct comparison on the same datasets and metrics.

\subsection{Streaming Loss Computation}

Because the vocabulary may be large (e.g.\ BPE vocabularies with $\sim
10^5$ tokens), computing logits for all time steps simultaneously can be
prohibitively expensive.

Instead, the loss is computed in a streaming fashion:
\begin{itemize}
    \item at each time step, logits of shape $(\text{batch} \times
    |\mathcal{V}|)$ are materialized,
    \item cross-entropy is computed immediately,
    \item logits are discarded before the next step.
\end{itemize}

This reduces peak memory usage and enables training on commodity GPUs
(e.g., 8GB, 16GB or 24GB VRAM) or even m4 Mac Minis with 32GB shared RAM, without specialized kernels.

\subsection{Comparison with Transformer Computation}

A standard Transformer represents tokens as embeddings added into a
residual stream and relies on self-attention to retrieve and mix past
information.
Order is injected primarily through positional encodings or rotary
embeddings.

In contrast, the Arrow model:
\begin{itemize}
    \item represents tokens as state-transforming operators,
    \item encodes order through non-commutative composition,
    \item stores history implicitly in the evolving state rather than
    via content-addressable retrieval.
\end{itemize}

This shifts the inductive bias from retrieval to transformation.
A token influences future predictions not by being attended to, but by
having permanently modified the proof context.

\subsection{Interpretation and Extensions}

From a logic programming perspective, the Arrow model can be viewed as a
differentiable abstract machine in which each token is an instruction
that mutates a hidden store.
From a Curry--Howard perspective, it is a numeric realization of proof
construction by successive implication elimination.

While the baseline model is strictly recurrent and sequential, it can be
extended with lightweight causal mixers or attention layers without
abandoning the core implicational semantics.
Such hybrids are left for future work.

We refer to the code in \url{https://github.com/ptarau/nextword/blob/main/arrow.py} for
the details of our implementation for both training and inference.

\section{Experiments}\label{exper}

Our experiments are organized around a guiding hypothesis:
\emph{left-nested implicational structure provides an order-sensitive representation
    that supports both symbolic retrieval (Section~\ref{info}) and neural next-token
    prediction via operator composition}.
To make the link between the logic-level retrieval mechanism and the learned
neural mechanism maximally transparent, we use a deliberately overfitted training
regime that turns the training corpus into a ``memorized'' store of sentences and
sentence fragments. This mirrors the Prolog dynamic database of assumptions used
by the logic retrieval procedure.

\subsection{Sanity checking the model: via Overfitting on word tokens}

Attempting to overfit a small document or batch—often called a "Sanity Check"—is a common and effective indicator for evaluating a transformer (or similar) neural architecture choice.
While standard overfitting is typically viewed negatively, in this specific context, it serves as a diagnostic tool for the following:
\BE
\I
Verification of Learning Capacity in a
Proof of Concept: If a transformer cannot overfit (achieve near-zero training loss) on a tiny dataset, it indicates the architecture lacks the necessary capacity or has a structural flaw.
Information Flow: Success confirms that the input features contain enough information for the model to extract patterns and that the architecture can technically represent those patterns.
\I Debugging Optimization and
Initialization Issues: Large transformers often struggle with convergence on small datasets due to poor initialization rather than architectural weakness. Techniques like DT-Fixup can resolve these issues, allowing deep models to succeed even where they previously failed.
\I Normalization Check: It helps identify if specific layers, like Layer Normalization, are correctly handling the scale of the data.
\I Identifying Architectural Efficiency
Complexity vs. Size: Observing how fast a model converges on a small sample can reveal if the architecture is unnecessarily large or complex for the intended task.
\I Inductive Bias: Transformers have less "built-in" knowledge (inductive bias) than architectures like CNNs, making them naturally prone to overfitting small data. If a specific transformer variant overfits too easily without learning meaningful relationships, it may lack the structural constraints needed for your domain.
\EE

\subsection{Data acquisition and preprocessing}
\label{sec:data}

\paragraph{Gutenberg acquisition.}
We obtain public-domain ASCII texts from Project Gutenberg
(\url{https://gutenberg.org}) and treat each document as a single large corpus.
For each document, we automatically download the plain-text UTF-8 variant when
available, remove the standard Gutenberg header/footer boilerplate, and retain
only the main body. This yields large, clean corpora without licensing friction.

\paragraph{Sentence sets and normalization.}
Each document is converted into a set of sentences, with one sentence per line,
to support two complementary pipelines:
(i) conversion into implicational formulas for Prolog-based retrieval, and
(ii) tokenization for neural training and evaluation.
We apply lightweight normalization (lowercasing and punctuation cleanup) to
reduce accidental duplication caused by formatting artifacts.
The resulting sentences serve as the \emph{ground truth store} for both Prolog
and Python retrieval experiments.

\subsection{Training protocol: learning on sentences and fragments}
\label{sec:training-protocol}

\paragraph{Fragments as implicational subformulas.}
Section~\ref{info} shows that retrieval in the logic setting can be driven by a
query that is a \emph{suffix subformula} (or prefix of such a suffix) of a stored
sentence formula. In the left-nested encoding, a contiguous token span
corresponds to a syntactically identifiable substructure. To mimic this in the
neural setting, we augment the training set with \emph{all contiguous fragments}
of each sentence, i.e., all subsequences $x_i,\dots,x_j$ that appear contiguously
in a sentence $x_1,\dots,x_n$.

Concretely, for each sentence token sequence $(x_1,\dots,x_n)$ we include all
fragments $(x_i,\dots,x_j)$ for $1 \le i \le j \le n$.
We deduplicate fragments globally across the document to avoid bias from repeated
lines (common in large novels and dialogue-heavy texts).

\paragraph{Why deliberate overfitting is useful here.}
This training regime is intentionally biased toward memorization: it turns the
model into a learned scoring function over the space of fragments and their
continuations. This is not presented as a generalization benchmark but as a
controlled setting that isolates the relationship between:
\begin{itemize}
    \item the \emph{logic retrieval} mechanism (unification/provability over
    implicational encodings), and
    \item the \emph{neural retrieval} mechanism (subsequence matching followed by
    a learned continuation score).
\end{itemize}
In other words, we are testing whether the Arrow-style recurrence can
consistently encode enough order information to support the same kind of
query-to-suffix completion that is straightforward in the implicational logic
representation.

\paragraph{Objective and optimization.}
Training uses the standard next-token cross-entropy objective with teacher
forcing: given a fragment $(x_1,\dots,x_T)$, the model predicts $x_{t+1}$ from the
state obtained after processing $x_{\le t}$.
We use AdamW optimization with a warmup schedule and (when training on CUDA)
mixed precision (fp16/bf16 where supported) to improve throughput.
Because the output projection to the vocabulary dominates compute for large
vocabularies, we compute the loss in a \emph{streaming} fashion over time steps:
at each step we materialize logits of shape $\text{batch}\times|\mathcal{V}|$,
compute cross-entropy, and immediately discard logits to keep peak memory
manageable on commodity GPUs.

\subsection{Generation and retrieval evaluation}
\label{sec:generation-eval}

We evaluate two complementary modes.

\subsubsection{Symbolic retrieval baseline (Prolog)}
\label{sec:prolog-retrieval-eval}

In the Prolog pipeline, each sentence is stored as an implicational formula in a
dynamic database of assumptions, and queries are posed as subformulas (or prefixes
of suffix subformulas) of these stored formulas.
Retrieval succeeds when the query matches a stored sentence structure under the
intended implication encoding; this corresponds to proving or reconstructing the
completion using intuitionistic principles described earlier.
This provides a crisp, interpretable baseline: \emph{retrieval is exact and
    structure-driven}.

\subsubsection{Neural retrieval mimic (Python)}
\label{sec:python-retrieval-eval}

The Python implementation mimics the logic retrieval process in two stages:

\paragraph{(1) Candidate enumeration by subsequence match.}
Given a user query consisting of a short word sequence $q=(q_1,\dots,q_m)$, we
scan the stored sentences (the same sentence set used for training) and collect
all occurrences where $q$ appears as a contiguous subsequence. Each occurrence
induces a candidate \emph{suffix completion} (from the match position to the end
of the sentence). This mirrors the way a suffix subformula of a stored
implicational chain determines a unique completion in the symbolic setting.

\paragraph{(2) Ranking by learned continuation score.}
For each candidate suffix, we compute a log-likelihood score under the trained
Arrow model for the continuation \emph{after} the matched query. Intuitively,
this score measures how compatible the candidate completion is with the learned
operator dynamics induced by the query prefix.
Candidates are ranked by this score, and we return the top-$k$ completions.
To avoid redundant answers, we deduplicate candidates by their generated text:
multiple identical suffixes arising from different corpus locations are returned
only once.

This neural retrieval mode can be seen as a ``soft'' analogue of the Prolog
retrieval: instead of logical entailment/unification, we use learned conditional
probability to select the most plausible completion among exact substring matches.

\subsubsection{Free continuation (pure neural generation)}
\label{sec:free-generation}

When the query does not appear as a contiguous subsequence in the stored
sentences (or when retrieval is disabled), we fall back to standard neural
generation: we initialize the recurrent state using the prompt tokens and then
sample one token at a time from the output logits.
We use either greedy decoding or temperature-controlled sampling; in addition,
beam search is useful for short completions in the word-level setting.
This mode tests whether Arrow-style operator composition alone can produce
fluent continuations beyond memorized retrieval.

\subsection{Scaling-up on very large documents}
\label{sec:scalingup}

A key practical question is whether both the symbolic and neural components
remain usable on corpora substantially larger than ``toy'' examples.

\paragraph{Large-document pipeline.}
For large novels (e.g., Tolstoy's \emph{War and Peace} or Proust's
\emph{The Guermantes Way}), the pipeline remains the same:
download Gutenberg text, clean boilerplate, segment into sentences, and build
(i) a Prolog database of implicational encodings and (ii) a neural training set
of sentences and fragments.
The main scaling pressure comes from corpus size and vocabulary size, not from
any change in the underlying algorithms.

\paragraph{Performance considerations: Prolog vs.\ torch inference.}
Both approaches scale in a predictable way:
\begin{itemize}
    \item \textbf{Prolog retrieval} scales with the size of the dynamic database
    and the indexing strategy used for matching subformulas. Because queries are
    structural, retrieval is typically fast once the database is built, and can
    be tuned with indexing on leading functors/atoms.
    \item \textbf{Neural scoring} scales linearly in sequence length per candidate
    and benefits substantially from GPU acceleration. In the Arrow baseline, the
    per-step recurrence is inexpensive, and the dominant cost is the vocabulary
    projection. Mixed precision and streaming loss reduce memory pressure,
    allowing training and inference on a single commodity GPU.
\end{itemize}

\paragraph{Experimental plan for large-scale runs.}
For large documents we report:
(i) corpus statistics (number of sentences, vocabulary size, number of unique
fragments under deduplication),
(ii) training throughput (tokens/sec) and peak memory usage,
and (iii) retrieval quality measured by whether a query prefix returns the
intended sentence completions.
We also contrast retrieval-first generation with free generation to separate
``memorized completion'' behavior from genuinely compositional continuation.

\paragraph{Summary.}
Overall, these experiments (with details shown in the next subsection) treat large documents as a stress test for the core
claim: the implicational encoding yields a natural notion of order-sensitive
substructure that supports both logic-level retrieval and a closely related
learned retrieval/scoring mechanism in the Arrow model.

\subsection{Scaling and Runtime Measurements}
\label{sec:runtime}

All experiments in this section use a maximum sentence length of $256$ words and an
augmentation regime that includes all contiguous fragments up to length $5$ (globally deduplicated).
This mirrors the logic-level setting of Section~\ref{info}, where queries correspond to subformulae
of stored implicational chains, and it matches the Python retrieval pipeline, where a query is a
contiguous subsequence whose best suffix completions are ranked by a learned continuation score.

\begin{table}[t]
    \centering
    \small
    \begin{tabular}{lrr}
        \toprule
        \textbf{Document} & \textbf{Training time} & \textbf{Notes} \\
        \midrule
        \texttt{the\_eyes}         & 6m 47.355s & 256-word sentences, fragments $\le 5$ \\
        \texttt{war\_and\_peace}   & 7m 52.297s & same settings \\
        \texttt{ulysses}          & 9m 12.960s & same settings \\
        \texttt{guermantes}       & 6m 43.017s & same settings \\
        \texttt{wizard\_of\_oz}    & 4m 20.580s & same settings \\
        \texttt{moby\_dick}        & 6m 47.355s & same settings \\
        \texttt{dracula}          & 5m 20.936s & same settings \\
        \texttt{cthulhu}          & 4m 43.505s & same settings \\
        \texttt{crystal}          & 5m 54.452s & same settings \\
        \bottomrule
    \end{tabular}
    \caption{End-to-end training runtimes on several Gutenberg documents under a fixed preprocessing
        and augmentation protocol: maximum sentence length $256$ words and maximum fragment length $5$
   on a Linux with 32GB of RAM with an RTX 3090 GPU with 24GB of VRAM.}  \label{tab:train-times}
\end{table}

\paragraph{Training throughput and scaling.}
Under the above fixed constraints, training completes in a few minutes per document
(Table~\ref{tab:train-times}).
The variation across texts is primarily driven by (i) the number of extracted sentences,
(ii) the number of unique fragments induced by each text after global deduplication,
and (iii) vocabulary size effects on the output projection and cross-entropy computation.
Importantly, the protocol is stable across a range of document sizes, including long novels,
suggesting that both the symbolic and neural pipelines can be scaled to book-length corpora.
To reduce RAM memory consumption from quadratic to linear we also limit the length of the
subsequence fragments to a small k, with k=5 in our experiments. This assumes that the queries
will be no longer then k words, which should be good enough for most exact subsequence queries.

\paragraph{Inference-time behavior.}
Inference in the retrieval-first Python pipeline (candidate enumeration by subsequence match
followed by continuation scoring) takes approximately $0.1$--$0.3$ seconds per query, and was observed
to be largely independent of the overall file size in these experiments.
This is consistent with the fact that, once candidate suffixes are identified, scoring requires only
short sequential evaluation over the candidate continuation (bounded by the maximum sentence length),
rather than processing the entire document.

\paragraph{Prolog database construction and query time.}
In the Prolog baseline, loading the dynamic database of sentence formulas from a preprocessed
\texttt{sent\_file} takes approximately $1$--$2$ seconds, while typical queries complete in roughly
$0.5$ seconds.
This supports the feasibility of using logic-level retrieval as an interpretable baseline even for
large documents: database initialization is fast, and query evaluation remains interactive.

\paragraph{Summary.}
Overall, these measurements indicate that (i) the fragment-augmented training setup is practical at
\url{https://gutenberg.org} very large novels scale, (ii) retrieval-style inference remains interactive, and (iii) Prolog-based
formula indexing provides a lightweight and interpretable retrieval baseline complementary to the
learned ranking/scoring mechanism.

\section{Discussion}\label{disc}
\paragraph{Limitations.}
The present work deliberately focuses on settings where the query/prefix is an \emph{exact} contiguous
substructure of a stored sentence (or fragment), so that both the Prolog retrieval procedure and the
Python mimic operate under exact matching assumptions. While this controlled regime is useful for
making the implicational semantics explicit, it does not address the robustness requirements of
modern language modeling, such as prompts containing misspellings, omissions, paraphrases, or
non-contiguous evidence.

\paragraph{Future work directions: towards a proof-theoretic account of transformers}
A central open challenge is to emulate and explain transformer-style inference using similar
logic-based techniques. In particular, transformers can perform next-token prediction even when the
prompt is noisy or partially specified; this suggests an inference mechanism that behaves like
\emph{approximate} matching and \emph{soft} hypothesis selection rather than exact unification.

\paragraph{Attention as approximate implication elimination.}
A useful starting point is to interpret attention as a family of conditional rules of the form
\[
(E \to w_{\text{before}}) \to w_{\text{after}},
\]
where $E$ denotes a (possibly distributed) context extracted from the prefix. This resembles a
modus-ponens pattern:
\[
E \to w_{\text{before}}
\qquad
(E \to w_{\text{before}}) \to w_{\text{after}}
\;\;\Rightarrow\;\;
w_{\text{after}}.
\]
In this reading, the model maintains a context-sensitive implication $E \to w_{\text{before}}$,
and uses learned implicational links to predict $w_{\text{after}}$.

\paragraph{Query--Key matching as implication selection.}
In standard attention, a query vector $q$ selects keys $k_i$ via a similarity score, and returns a
weighted combination of values $v_i$. A proof-theoretic analogue would treat the memory as a set of
implicational clauses indexed by keys:
\[
\{\, k_i \Rightarrow v_i \,\}_{i=1}^m,
\]
and interpret querying as selecting (softly) which implication instances are applicable. Informally,
one may view a single head as computing something akin to:
\[
(q \Rightarrow k_i) \Rightarrow v_i,
\]
i.e., \emph{if the query establishes (approximately) the key, then the associated value follows}.
The attention softmax then becomes a differentiable surrogate for choosing which implication instances
to apply.

\paragraph{Matrix view: lower-triangular implication structure.}
Causal self-attention enforces a strict lower-triangular dependence structure: token $t$ may depend
only on positions $< t$. This can be seen as a time-indexed family of implication links that only
flow forward. If one restricts attention weights to $\{0,1\}$, the attention matrix becomes a
lower-triangular adjacency matrix of a proof graph, where an edge $i \to t$ indicates that evidence
at position $i$ participates in justifying the prediction at $t$. Relaxing to real-valued weights
yields a \emph{weighted} proof graph, suggesting connections to proof nets and probabilistic proof
systems.

\paragraph{Feed-forward blocks as implicational pipelines.}
Transformer MLP blocks can be interpreted as learned compositions of transformations acting on
residual states. A logic-inspired view is to treat them as implicational pipelines
\[
p_1 \to p_2 \to \cdots \to p_n \to \textit{out},
\]
where intermediate propositions $p_i$ correspond to latent predicates/features computed by the MLP.
While this is schematic, it suggests that MLP layers may be expressible as structured sequences of
implication introductions/eliminations over latent vocabularies.

\paragraph{Research direction.}
The overarching goal is to develop a constructive account of {\em transformer inference} in which:
(i) keys index reusable implicational clauses,
(ii) queries establish applicability of those clauses in a graded fashion,
(iii) values represent the resulting derived facts, and
(iv) multi-head composition corresponds to parallel (but interacting) proof steps.
Such a view would generalize the exact, left-nested implicational retrieval studied in this paper
to an approximate and distributed setting capable of handling noisy prompts.

\section{Related Work}\label{rel}

\subsection{Transformers and Positional Mechanisms}

Transformers \cite{transfo} represent sequences using content-based interactions computed by
scaled dot-product attention, followed by weighted averaging of value vectors.
In its basic form, the attention kernel depends on pairwise similarity between token-derived
queries and keys, and is therefore not itself an order-sensitive algebraic operation:
without additional structure, attention is permutation-invariant with respect to the input positions.
Consequently, Transformers inject order through explicit positional mechanisms, including
absolute position encodings \cite{transfo},
relative-position representations \cite{shaw2018relative},
and long-context schemes combining recurrence and position handling (e.g., Transformer-XL \cite{daixl2019}).

A large body of work refines positional structure while keeping attention as the core interaction.
Rotary Position Embeddings (RoPE) \cite{su2021roformer} incorporate relative offsets by rotating
query/key representations; ALiBi \cite{press2021alibi} biases attention logits as a linear function
of distance, improving length extrapolation.
These methods emphasize that, in the transformer family, order is primarily enforced by
positional augmentation of an otherwise similarity-driven (dot-product) interaction.

Arrow models depart from this design choice: rather than adding positional information to a
similarity operator, they encode order directly via non-commutative composition.
A prefix is interpreted as a left-nested implicational chain acting on a state,
so the model's sequential backbone is intrinsically order-sensitive.
In this sense, Arrow shifts the inductive bias from \emph{positionalized similarity}
to \emph{operator composition}.

\subsection{State-Space Models, Selective SSMs, and Long Convolutions}

State-space sequence models (SSMs) provide an alternative to attention in which long-range
dependencies arise from structured linear dynamics.
S4 \cite{s4gu2021} and follow-on work show how to parameterize and compute SSMs efficiently,
achieving strong performance on long-range benchmarks while avoiding quadratic attention costs.
Mamba \cite{gu2024mambalineartimesequencemodeling} 
introduces \emph{selective} (input-conditioned) state-space updates that improve
modeling of discrete, information-dense data such as language, and scales linearly in sequence length.
Recent analyses further connect attention and SSM computation via structured matrix views
(\emph{structured state space duality}) \cite{daogu2024ssmduality}, clarifying when attention-like
effects can emerge from SSM-style operators.

In parallel, long-convolution architectures such as Hyena \cite{poli2023hyena} replace attention
with implicit long filters and data-controlled gating, offering another sub-quadratic route to
long-context modeling.

Arrow differs from classical SSMs in that its recurrence is explicitly \emph{token-dependent} and
operator-valued, rather than governed by a fixed global transition/kernel.
This places Arrow closer to the family of input-conditioned recurrent models, while sharing the
SSM motivation of making long-context computation efficient and structurally grounded.

\subsection{Multiplicative and Second-Order RNNs; Hypernetworks}

The Arrow update is mathematically closest to recurrent models in which the input chooses a
state transition operator.
Sutskever, Martens, and Hinton \cite{suts11} introduce a multiplicative RNN for language modeling in which
the current input modulates the hidden-to-hidden transformation via a low-rank factorization,
and show strong results when combined with Hessian-free optimization; see also Martens and Sutskever
for the Hessian-free training method \cite{martens2011hf}.
This line of work can be read as learning a \emph{library of basis transformations} whose mixture is
selected by the current symbol, closely mirroring Arrow's shared-basis, token-gated operator form.

Related ideas appear in Multiplicative Integration \cite{wu2016mi}, which improves gradient flow by
multiplicatively coupling different information sources inside RNN blocks, and in HyperNetworks
\cite{ha2016hypernetworks}, where one network generates (or modulates) the weights of another.
All these models use multiplicative structure to obtain input-conditioned transformations.

The conceptual difference is that Arrow is derived from an intuitionistic-implicational reading of
prefixes (tokens as operators under a Curry--Howard semantics), rather than introduced primarily as a
gating or interaction mechanism.
This derivation suggests a proof-theoretic interpretation of learned transitions:
the model state summarizes a constructive context, and next-token prediction corresponds to
a form of learned modus-ponens-like continuation in the embedded operator algebra.

\subsection{Recursive and Tool-Using Inference for Long Contexts}

Zhang, Kraska, and Khattab propose \emph{Recursive Language Models} (RLMs), an inference-time strategy
for handling prompts far longer than the model’s context window by treating the prompt as an external
environment that the model can examine, decompose, and recursively re-query \cite{zhang2025rlm}.
Rather than modifying the base model architecture, RLMs provide a structured control flow that
iteratively selects relevant snippets and composes partial results, achieving strong performance on
long-context tasks at comparable or lower cost per query. Conceptually, this line of work is aligned
with our use of structured compositions of inference steps: it emphasizes that \emph{explicit
    program-like control} (recursion, decomposition, self-calls) can be as important as the internal
parametric mechanism for realizing deep, multi-step reasoning. :contentReference[oaicite:0]{index=0}
In \cite{tarau2023automation}  a logic-programming inspired method for \emph{automating goal-driven LLM dialog
    threads} by organizing reasoning as a bounded-depth recursive exploration of OR-alternatives and AND-expansions,
explicitly derived from Horn-clause (SLD-resolution-like) control.
The method synthesizes intermediate prompts summarizing the evolving trace, uses similarity/oracle
checks to prune and validate steps, and aggregates results via a generated Horn program’s minimal
model. This work complements our approach by showing how LP control principles can structure and
validate LLM inference externally, whereas our present paper focuses on how implicational structure
itself motivates an internal, compositional view of next-token prediction and multi-token continuation.

\subsection{Logic Programming and Neuro-Symbolic works related to LLM Reasoning}
A growing line of neuro-symbolic work treats an LLM primarily as a \emph{semantic parser} that maps natural language into a formal language (FOL, ASP, Prolog), and then delegates multi-step reasoning to a symbolic engine with explicit proof traces.
LINC \cite{olausson2023linc} converts premises and hypotheses to first-order logic and discharges validity with an external theorem prover, yielding improved reliability and interpretability on logical reasoning benchmarks.
Similarly, Yang et al.\ \cite{yang2023coupling} couple few-shot LLM parsing with reusable Answer Set Programming (ASP) knowledge modules to obtain robust reasoning across multiple QA tasks without retraining, emphasizing the modularity and inspectability of the symbolic component.
In \cite{flops24} deep step-by step reasoning in an LLM dialog thread is automated by recursively exploring alternatives (OR-nodes) and expanding details (AND-nodes) up to a given depth.
The algorithm is derived from a simple recursive descent implementation of a Horn Clause interpreter. Semantic similarity to ground-truth facts or oracle advice from another LLM instance is used to restrict the search space and validate the traces of justification steps returned as focused answers.
In the ASP direction, Ishay et al.\ \cite{ishay2023asp} show that LLMs can generate nontrivial ASP encodings for logic puzzles (often with simple, human-correctable errors), enabling a workflow where correctness is enforced by a solver rather than implicit pattern completion.

More recent logic-programming work pushes this integration further in two complementary ways.
First, LLM2LAS \cite{santana2025llm2las} combines LLM-based extraction of semantic structure with \emph{rule induction} (ILASP/LAS), learning interpretable rules that are then executed by an ASP solver for story-based question answering.
Second, Prolog-oriented frameworks such as LoRP \cite{di2025lorp} translate natural language into (extended) Prolog and use an external Prolog interpreter for verifiable inference, often adding validation and voting to increase robustness under long inference depth.
Finally, solver-guided training and decoding is emerging as a systematic mechanism for reducing DSL/program-generation errors: In \cite{schrader2025solverloop} they use an ASP \emph{solver-in-the-loop} feedback to label partial continuations as accepted/rejected and fine-tune the LLM accordingly.

These approaches are complementary to Arrow: rather than outsourcing deduction to an external prover/solver, Arrow aims to \emph{internalize} an implication-shaped constructive update in the neural dynamics, while retaining a semantic link to proof-theoretic structure.

\section{Conclusion}\label{conc}

We presented the Arrow Language Model, grounded in \emph{left-nested} intuitionistic implication.
A formal correspondence links implication depth to recurrent computation, explaining why order
is intrinsic to the architecture. This yields a logic-first foundation for next-token prediction
distinct from similarity-based attention and from convolution/state-space alternatives.
For the reader curious to explore independently our open source code, a quick way to 
start are the simple commands shown at \url{https://github.com/ptarau/nextword/blob/main/README.md} .

\bibliographystyle{tlplike}

\bibliography{tarau,llm,ml,proglang, theory}

\end{document}